\documentclass[12pt]{report}

\usepackage[english]{babel}

\usepackage[a4paper,top=2cm,bottom=2cm,left=3cm,right=3cm,marginparwidth=1.75cm]{geometry}

\usepackage{amsmath}
\usepackage{fullpage}
\usepackage{amssymb}
\usepackage{amsthm}
\usepackage{bbm}
\usepackage{graphicx}
\usepackage{csquotes}
\usepackage[colorlinks=true, allcolors=blue]{hyperref}
\usepackage[style=alphabetic-verb]{biblatex}
\usepackage{cleveref}
\newtheorem{theorem}{Theorem}[section]
\newtheorem{lemma}[theorem]{Lemma}
\newtheorem{corollary}[theorem]{Corollary}
\newtheorem{remark}[theorem]{Remark}
\newtheorem{assumption}{Assumption}

\title{{M2 Advanced Mathematics Report:}\\
{The Graphical Nadaraya-Watson Estimator on Latent Position Models}\\
\includegraphics[width=0.8\textwidth]{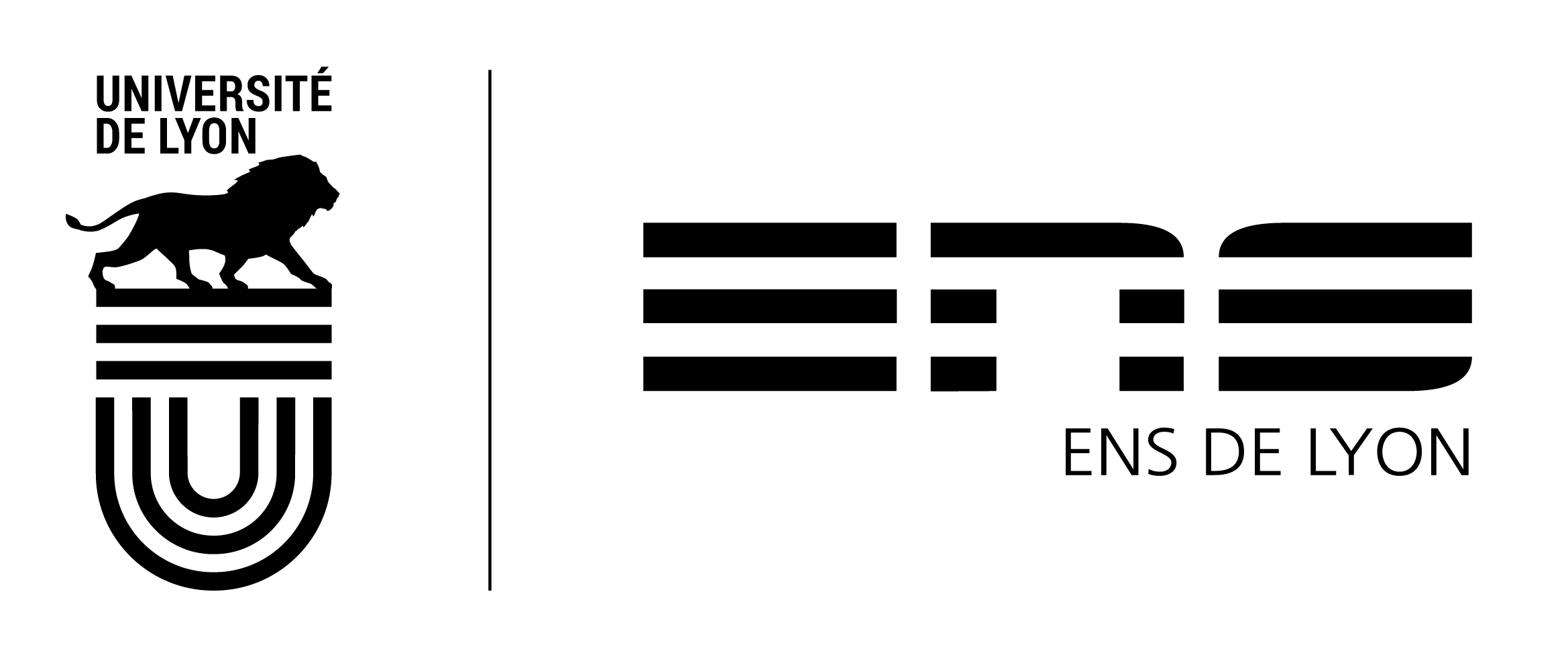}
}

\author{Martin Gjorgjevski  \\
Under the supervision of \\  Nicolas Keriven (CNRS, Gipsa-lab),\\ Simon Barthelmé (CNRS, Gipsa-lab) \\ and Yohann de Castro (Ecole Centrale Lyon)}
\date{August 2022}
\DeclareMathOperator\supp{supp}

\DeclareMathOperator{\Var}{Var}
\DeclareMathOperator{\Bias}{Bias}
\usepackage{microtype}

\begin{document}

\maketitle
\begin{abstract}
Given a graph with a subset of labeled nodes, we are interested in the quality of the averaging estimator which for an unlabeled node predicts the average of the observations of its 
labeled neighbours. We rigorously study concentration properties, variance bounds and risk bounds in this context. While the estimator itself is very simple we believe that our results will contribute towards the theoretical understanding of 
learning on graphs through more sophisticated methods such as Graph Neural Networks.
\end{abstract}

\tableofcontents

\chapter{Introduction}

Given a undirected graph on $n+1$ vertices (e.g. nodes represent people) 
and adjacency matrix $A=[a_{i,j}]$ (e.g. edges represent social relationships) where all but the $(n+1)$-st node
have labels $y_i$ (e.g. salary, living expenses, etc), the graph regression problem adresses prediction of the
(continuous valued) label $y_{n+1}$ of the remaining node. While there are various sophisticated designs of 
Graph Neural Networks \cite{GCN,GraphSAGE,GAT,GIN} which can tackle this problem, little has been done 
in terms of statistical analysis of \textit{any} potential solution in the context of random graph models.
In this paper we will consider the simplest estimator, which for the missing label of node $n+1$ is taking the average over all of its neighbours, i.e.
\begin{equation}
\label{meanagg}
    \hat{y}_{n+1}=\frac{\sum_{j=1}^n{y_ja_{j,n+1}}}
    {\sum_{j=1}^na_{j,n+1}
    }
\end{equation}
To our knowledge, the statistical properties of this estimator have \textit{not} been studied in the statistical
or machine learning literature. The main reason for this is the lack of statistical modelization of the data generating process. Indeed, without
imposing a stochastic structure on the graph, key quantities such as sample complexities and generalization bounds are not properly defined. 
To conduct our analysis we will work with a \textit{random graph model} known as the Latent Position Model \cite{Hoff}, where to each node one associates a \textit{latent} position in space. 
Due to the nature of this model, the estimator \eqref{meanagg} will resemble the \textit{Nadaraya-Watson} estimator,
a popular regression estimator in the nonparametric estimation literature \cite{Tsybakov}. 
For this reason we decide to title it the \textbf{Graphical Nadaraya-Watson} (GNW) estimator. We show that under classical assumptions 
on the regression function $f$ and the kernel $k$, the \textit{Graphical Nadaraya-Watson} (GNW) esimator achieves the same rates for the pointwise and integrated risk 
as those of the \textit{Nadaraya-Watson} estimator. The major difference between NW and GNW estimators is that the NW is more expressive in the sense 
that it comes with a tunable parameter (\textit{bandwidth}) while for the GNW there are no tunable parameters. As such, the performance of GNW 
depends on a bandwitdh which is \textit{not user-chosen}. In the asymptotic regime there is a certain range of values for which NW has low prediction error.
Practically speaking, we show that if the \textit{latent} parameter of GNW falls in the range for which the NW estimator has low error of prediction, 
then GNW will achieve that same error (within a multiplicative constant). Whereas most of the methods discussed in the literature require degree of \textit{logarithmic} order i.e. 
$d_n=\omega(\log(n))$ \cite{Lei_2015,Oliviera}, we show that the variance of GNW will converge to zero for \textit{all} regimes of sparsity, i.e. the \textit{only} requirement is that $d_n=\omega(1)$. The pooling procedure used in equation \eqref{meanagg} is known as \textit{mean aggregation}
in the Graph Neural Networks literature \cite{GCN}. Therefore although not immediately, our methods can extend to study statistical properties 
of Graph Neural Networks.
\section{Nonparametric Regression}
The (nonparametric) regression problem can be stated as estimating a \textit{regression}
function $f\colon\mathbb{R}^d\to\mathbb{R}$ given noisy measurements 
\begin{equation}
\label{labels}
    Y_i=f(X_i)+\epsilon_i
\end{equation}
where $f\colon\mathbb{R}^d\to\mathbb{R}$ with $||f||_{\infty}\leq B$, $\epsilon_i$ additive centered noise with finite variance. 
One is also given the data points $X_1,...,X_n$ which are either deterministic (fixed design) or random i.i.d.
samples from a distribution with density $p$ (random design). Typically one needs a certain order of regularity of the regression function $f$ 
(and on the distribution $p$ as well, in the case of random design). One of the many ways to measure regularity is the \textit{Hölder class} $\Sigma(\beta,L)$ \cite{Tsybakov} which for $0<\beta\leq 1$ and $L\geq 0$ is given by
\begin{equation}
    \label{holderclass}
    \Sigma(\beta,L)=\{f\colon\mathbb{R}^d\to\mathbb{R}|\hspace{5pt} \textit{for all} \hspace{5pt} x,z\in\mathbb{R}^d, \hspace{5pt} |f(x)-f(z)|\leq L||x-z||^{\beta}\}
\end{equation}
By assuming Hölder continuity on all partial derivatives up to a certain order,
one can get families of functions with higher \textit{regularity} ($\beta>1$), for which faster rates of convergence can be obtained \cite{Tsybakov,Gyofri}.
For our purposes, the assumption of \textit{Hölder} continuity with $0<\beta\leq 1$ will suffice. 

\begin{figure}[h!]
    \centering
    \includegraphics[width=0.4\textwidth]{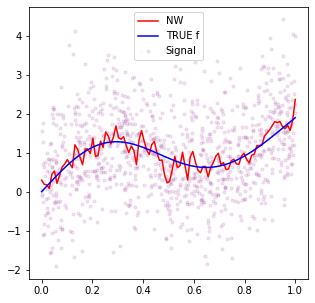}
    \includegraphics[width=0.4\textwidth]{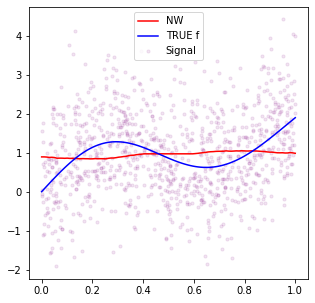}
    \caption{Bias-Variance Tradeoff: n=1000 points are sampled uniformly and independently on $[0,1]$ and Gaussian noise with variance $\sigma^2=1$ is added. Left: $\hat{f}_{NW}$ estimator with $h=0.01$. Right: $\hat{f}_{NW}$ estimator with $h=0.6$}
    \label{fig_nw_sensitivity}
\end{figure}
A clasical approach for the regression problem is the weighted average
\textit{Nadaraya-Watson} estimator \cite{Tsybakov,devroye}
\begin{equation}
\label{NW}
    \hat{f}_{NW}(x)=\begin{cases}
        \frac{\sum_{i=1}^n Y_ik(x,X_i)}{\sum_{i=1}^n k(x,X_i)} \quad &\text{if}\, \sum_{i=1}^n k(
        x,X_i)\neq 0\\
        0 \quad &\text{otherwise}\\
    \end{cases}
\end{equation}
Here, $k(x,z)=K(\frac{x-z}{h})$ depends on the \textit{bandwidth} $h$ which controls the scale on which the data is being averaged. This parameter needs to be chosen carefully,
as too small values of $h$ produce estimates of high variance, while too large values of $h$ give
highly biased estimators, an instance of the \textit{Bias-Variance tradeoff}, a well known phenomenon in statistics (see Figure \eqref{fig_nw_sensitivity}). 
There are two main measures of statistical performance for NW \eqref{NW}, the \textit{pointwise} and \textit{integrated risk}.
For a given point $x\in\mathbb{R}^d$, the pointwise risk is given by 
\begin{equation}
\label{pointwise_risk}
    \mathcal{R}(\hat{f}_{NW}(x),f(x))=\mathbb{E}[(\hat{f}_{NW}(x)-f(x))^2]
\end{equation}
where the expectation is taken over the noise and the 
data points $X_1,...,X_n$ for the random design setting (only over the noise for the fixed design). It is also known as \textit{mean squared error (MSE)}. 
This metric is local in the sense that it only captures statistical information for a particular point. 
A metric that captures global statistical information is the \textit{integrated risk} given by 

\begin{equation}
\label{integrated_risk_nw}
    \mathcal{R}(\hat{f}_{NW},f)=\int\mathcal{R}(\hat{f}_{NW}(x),f(x))p(x)dx
\end{equation}
The integrated risk is also known as \textit{mean integrated squared error (MISE)} and can be interpreted
as the risk for a new random variable $X$ with density $p$, independent from the data $X_1,...,X_n$. 
\begin{figure}[h!]
    \centering
    \includegraphics[width=0.7\textwidth]{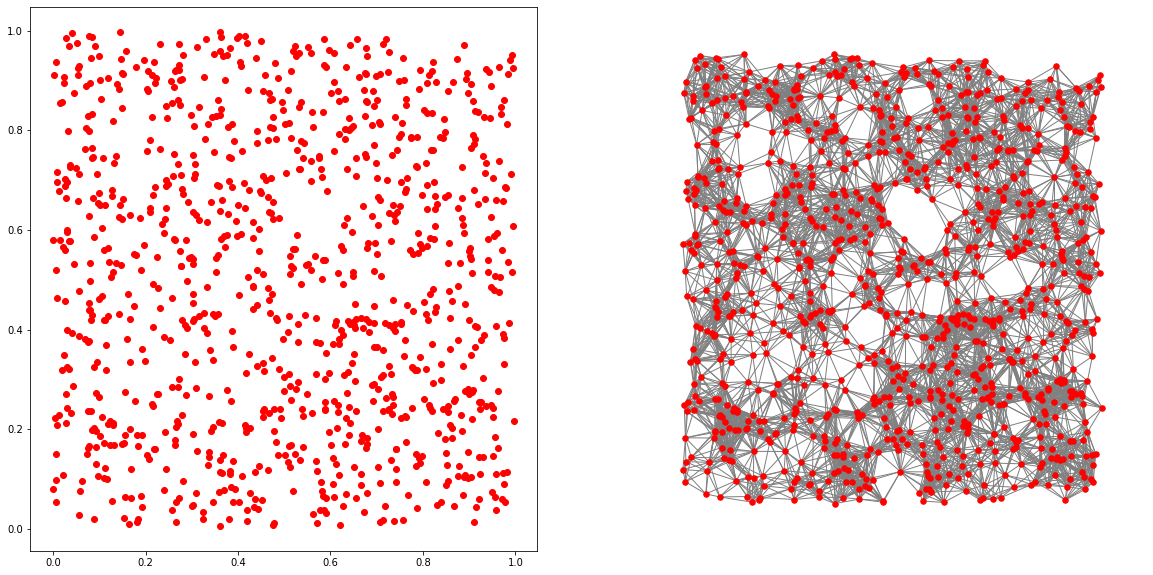}
    \caption{Random Geometric Graph with $n=1000$ uniformly sampled points on the unit square with average degree  $\log n$}
    \label{Rgg_fig}
\end{figure}
\section{Background on the Latent Position Model} The Latent Position Model (LPM) \cite{Hoff} is a generative model which
generates a random graph on $n$ nodes in two stages. First, a sample of $n$ i.i.d. \textit{latent} variables $X_i\in\mathbb{R}^d$ with density $p$ is drawn. Second, for each pair of nodes $i,j$ a Bernoulli
variable with parameter $k(X_i,X_j)$ determines if there is an edge between nodes $i$ and $j$. Here, $k$ is a symmetric kernel on $\mathbb{R}^d$ taking
values in $[0,1]$. The edge generating Bernoulli variables are conditionally independent given the latent variables. 
Intuitively we are more likely to observe an edge between two nodes which have positions
that are similar with respect to $k$. When $k$ is a convolutional kernel as in the NW estimator \eqref{NW}, edges are likely to occur
between nodes whose latent positions are nearby in the latent space. As an example, when $k(x,z)=\mathbb{I}(||x-z||\leq h)$, 
and NW and GNW coincide. This is the random geometric graph \cite{PENROSE} (see Figure \eqref{Rgg_fig}).
By allowing for discontinous kernels, LPMs can instantiate a model
 with intrinsic community structure known as Stochastic Block Model (SBM) \cite{Holland1983StochasticBF}. 
Despite lack of attention to the graph regression problem, classification has been adressed in the context of LPMs \cite{Tang}. On the other hand, there is a signifacnt literature for
clustering \cite{Snijders,Abbe} in SBMs. In \cite{Arias-Castro} the authors discuss recovering latent positions via graph distances. 
As large graphs in the real world tend to be sparse \cite{Albert}, 
a significant effort in the community detection literature is dedicated to
understanding statistical properties of graphs with low expected degrees \cite{Oliviera,Lei_2015,Levina-Vershynin}. In such frameworks one 
considers asymptotic regimes where 
\begin{equation}
\label{sbm_type_kernel}
k_n(x,z)=\alpha_n K(x,z)
\end{equation}
with $K$ being a fixed kernel and $\alpha_n\to 0$ as $n\to\infty$. The scaling $\alpha_n$ 
determines the sparsity of the graph where $n\alpha_n$ is interpreted as the expected degree of the graph. The parameter $\alpha_n$ is \textbf{not user-chosen} paramater, i.e.
\textbf{not known} to the statistician. In this sense, we will instead consider asymptotic
regimes with 
\begin{equation}
    \label{kernel_definition}
    k_n(x,z)=\alpha_nK(\frac{x-z}{h_n})
\end{equation}
where $K\colon\mathbb{R}^d\to [0,1]$ is compactly supported, $0<\alpha_n\leq 1$ and  
$h_n>0$, with $\alpha_n,h_n$ being parameters that are \textit{not user-chosen}. 
We emphasize again that the main difference between the setup for NW \eqref{NW} and GNW \eqref{meanagg} is the freedom 
to \textit{choose} the bandwidth $h_n$: this choice is up to the user for NW, for GNW it is not.
As $k_n(x,z)\leq \alpha_n$, the factor $0<\alpha_n\leq 1$ dictates the sparsity of the graph. Note that  
a LPM generated by Equation \eqref{sbm_type_kernel} is equal (in distribution) to a random graph obtained from a \textit{Bernoulli percolation} with parameter $\alpha_n$ on a LPM generated with kernel $K(x,z)$. 
\section{Framework and Notation} 
We observe a random graph with $n+1$ nodes sampled according to a LPM and assume that for all nodes but the last
there is a label of the form \eqref{labels}. Conditionally on node $n+1$ having latent position $x\in\mathbb{R}^d$, we write $a(x,X_i)$ for the indicator of an edge between the node $n+1$ and node $i$. 
\begin{equation}
\label{gnw_def}
\hat{f}_{GNW}(x)=\begin{cases}
    \frac{\sum_{i=1}^n Y_ia(x,X_i)}{\sum_{i=1}^n a(x,X_i)} \quad &\text{if}\, \sum_{i=1}^n a(x,X_i)\neq 0\\
    0 \quad &\text{otherwise}\\
\end{cases}
\end{equation}
Note that the only edges of interest for the Graphical Nadaraya-Watson estimator are those adjacent to $X$, thus we will not be concerned with the rest of the edge variables $a(X_i,X_j)$. 
From the modeling assumptions it follows that the indicator of an edge between (nodes associated to) $x$ and $X_{i}$ is given by 
\begin{equation*}
    a(x,X_{i})=\mathbb{I}[(U_i\leq k(x,X_i))]
\end{equation*}
where $U_1,...,U_n$ are uniform variables on $[0,1]$, such that $(X,X_1,...,X_n,U_1,...,U_n,\epsilon_1,....,\epsilon_n)$ are jointly independent random variables on a common probability space $(\Omega,\mathcal{F},P)$. 
While our results are nonasymptotic, we will frequently give comments on the asymptotic behavior of GNW \eqref{gnw_def}. For that reason we 
assume that the kernel $k$ depends on the sample size $n$. Similarly to the of pointwise risk \eqref{pointwise_risk} of NW \eqref{NW},
we consider \textit{pointwise risk} for GNW \eqref{gnw_def} by
\begin{equation}
    \label{formula}
    \mathcal{R}(\hat{f}_{GNW}(x),f(x))=\mathbb{E}[(\hat{f}_{GNW}(x)-f(x))^2]
\end{equation}
where the expectation is taken over all random variables appearing in the model 
(\textit{edge randomness, latent positions and noise}). Our main result is a bound on 
the \textit{integrated risk} 
\begin{equation}
\label{random_point_risk}
   \mathcal{R}(\hat{f}_{GNW},f)=\int \mathcal{R}(\hat{f}_{GNW}(x),f(x))p(x)dx
\end{equation}
The approach taken in this paper is to bound \eqref{formula} for all $x\in Q=\supp{(p)}$ and then to integrate the result
to obtain a bound on \eqref{random_point_risk}. 
\paragraph{Notation} We denote the indicator of a set $S$ by $\mathbb{I}(S)$, the Lebesgue measure on $\mathbb{R}^d$ by $m$ and the volume of the unit ball in $\mathbb{R}^d$ by $v_d$. The standard Euclidean distance between $x,y\in\mathbb{R}^d$ is denoted by $||x-y||$. 
The local edge parameter and the local degree  at a point $x\in\mathbb{R}^d$ are given by
\begin{equation}
\label{local_degree}
    c_n(x)=\int_{\mathbb{R}^d} k_n(x,z)p(z)dz, \hspace{20pt} d_n(x)=nc_n(x)
\end{equation}
respectively. For $x\in\mathbb{R}^d$, 
we define the operators $T_{k_n}(\cdot,x)$ and $b_n(\cdot,x)$ on the set of bounded and measurable functions by
\begin{equation}
\label{b_n}
    T_{k_n}(f,x)=\int f(z)k_n(x,z)p(z)dz, \hspace{20pt} b_n(f,x)=\begin{cases}
    \frac{T_{k_n}(f,x)}{c_n(x)} \quad &\text{if}\, c_n(x)>0\\
    0 \quad &\text{otherwise}\\
\end{cases}
\end{equation}
Finally, we denote by $Q$ the support of the distribution $p$
\begin{equation*}
    Q=\{x\in\mathbb{R}^d|\textit{ for all } r>0, \int_{B_r(x)} p(z)dz>0\}
\end{equation*}
\section{Outline}
We will follow a bias-variance decomposition  inspired approach.
For $x\in\mathbb{R}^d$ we introduce the \textit{variance proxy} at $x$ by
\begin{equation}
\label{variance_term}
    v_n(x)=\mathbb{E}[(\hat{f}_{GNW}(x)-b_n(f,x))^2]
\end{equation}
and the \textit{bias proxy} at $x$ by
\begin{equation}
\label{bias_term}
b_n(x)=b_n(f,x)-f(x)
\end{equation}
We remark that the variance and bias proxies introduced in \eqref{variance_term} and \eqref{bias_term} respectively do not correspond exactly to the classical statistical definitions of variance $\mathbb{E}[(\hat{f}-\mathbb{E}(\hat{f}))^2]$ and bias $\mathbb{E}[\hat{f}(x)]-f(x)$.
\subsection*{Chapter \ref{chap2}}
In this chapter we work with a general Latent Position Model. We use probabilistic tools to study the statistical behavior 
 properties of $\hat{f}_{GNW}$. In Section \ref{concentration_inequalities} we use concentration inequalities to show that if the noise variables $\epsilon_i$ are bounded in absolute value by $\sigma$, then 
\begin{equation*}
    \mathbb{P}(|\hat{f}_{GNW}(x)-b_n(f,x)|\geq\delta)\leq e^{-C\delta^2d_n(x)}
\end{equation*}
where $C$ depends on the boundedness constants $B$ and $\sigma$, of $f$ and $\epsilon_1$, respectively. In Section \ref{sharp_variance_bounds} we
study the variance term (\ref{variance_term}) at the point $x\in\mathbb{R}^d$. We prove that under finite second moment assumptions on the noise, 
\begin{equation*}
 \mathbb{E}[(\hat{f}_{GNW}(x)-b_n(f,x))^2]\leq\frac{c_2}{d_n(x)}
\end{equation*}
where $c_2$ depends on the boundedness constant $B$ of $f$ and the variance of the noise $\sigma^2$. This proof is the most technical part of the report.
It relies on Bernstein's concentration inequality and a technique specialized to Bernoulli variables which we call the \textbf{decoupling trick}. 
For a formal statement of this result we refer to Theorem \ref{variance_thm}. 
Finally, in Section \ref{expectation_computation} we derive an explicit value for 
$\mathbb{E}[\hat{f}_{GNW}(x)]$ using the decoupling trick.
\subsection*{Chapter \ref{bias_risk_control}} 
We focus on Latent Position Models with convolutional kernels, i.e. $k_n(x,z)=\alpha_nK(\frac{x-z}{h_n})$. 
Under the classical assumption of Hölder continuity on the regression function $f$ and density $p$,
we control the bias term (\ref{bias_term}). We show connections between the degree $d_n(x)$ and the bandwith $h_n$ under a geometrical condition on the 
support of the distribution on the latent positions 
which we call the \textbf{measure-retaining property}. Finally, we establish, under suitable assumptions, rates for the integrated risk of GNW \eqref{random_point_risk} similar to
those of NW \eqref{integrated_risk_nw}.

\paragraph{Note} Although all of our results hold for any $n\in\mathbb{N}$, 
we will often comment on asymptotic behaviors. 
This is the reason why we use the notations $k_n$, $c_n(x)$ and $d_n(x)$ for the kernel, 
the local connection parameter and the local expected degree respectively. 
Formally when we make an asymptotic comment we have in mind a sequence of graphs $(G_n)_{n\in\mathbb{N}}$ 
such that $G_n$ is sampled from a LPM on $n+1$ nodes, with kernel $k_n$ and with density of latent points $p$. 
With such a sequence of graphs one should emphasize that we have a triangular array of latent variables
\begin{equation*}
\begin{split}
    &X_{1,1},X_{1,2}\\
    &X_{2,1},X_{2,2},X_{2,3}\\
    &...\\
    &X_{n,1},X_{n,2},...,X_{n,n+1}\\
    &...
\end{split}
\end{equation*}
such that in any row the variables are i.i.d, with density $p$ and $G_n$ has $X_{n,1},....,X_{n,n+1}$ as latent variables. Similarly such triangular sequences exist for the noise variables $(\epsilon_{m,n})_{m\leq n+1}$ and the edge generating uniform variables $(U_{m,n})_{m\leq n+1}$. 
Throughout this report we assume that $f\colon\mathbb{R}^d\to\mathbb{R}$ is a bounded function with $||f||_{\infty}\leq B$.
\chapter{Statistical properties of GNW}
\label{chap2}
In this chapter we study the concentration properties of $\hat{f}_{GNW}(x)$. Our goals are to establish concentration rates i.e. bounds on $\mathbb{P}(|\hat{f}_{GNW}(x)-b_n(f,x)|\geq\delta)$ and to bound the variance proxy (\ref{variance_term}). 
As a byproduct of our methods we also compute the expectation $\mathbb{E}[\hat{f}_{GNW}(x)]$. Note that if $c_n(x)=0$ then $a(x,X_i)=0$ as $a(x,X_i)$ is a Bernoulli variable with probability
of succsses $c_n(x)$. Consequently by definition \eqref{gnw_def}, $\hat{f}_{GNW}(x)=0$. To avoid such trivialities, we assume $c_n(x)>0$ (hence $d_n(x)>0$).
\section{Concentration Inequalities Approach}
\label{concentration_inequalities}
Our goal in this section is to bound the probability $\mathbb{P}(\hat{f}_{GNW}(x)-b_n(f,x)|\geq \delta)$. We begin by observing that when $\sum_{i=1}^n a(x,X_i)>0$, we have
\begin{equation*}
\begin{split}
    \hat{f}_{GNW}(x)-b_n(f,x)&=\frac{\sum_{i=1}^n Y_ia(x,X_i)}{\sum_{i=1}^na(x,X_i)}-b_n(f,x)\\
    &=\frac{\sum_{i=1}^n[Y_i-b_n(f,x)]a(x,X_i)}{\sum_{i=1}^n a(x,X_i)}\\
    &=\frac{\sum_{i=1}^n[f(X_i)-b_n(f,x)]a(x,X_i)}{\sum_{i=1}^n a(x,X_i)}+\frac{\sum_{i=1}^n\epsilon_ia(x,X_i)}{\sum_{i=1}^n a(x,X_i)}
\end{split}
\end{equation*}
Moreover,
\begin{equation*}
\begin{split}
    \mathbb{E}([f(X_i)-b_n(f,x)]a(x,X_i))&=\mathbb{E}[f(X_i)a(x,X_i)]-b_n(f,x)\mathbb{E}[a(x,X_i)]\\
    &= T_{k_n}(f,x)-b_n(f,x)c_n(x)\\
    &=0
\end{split}
\end{equation*}
Similarly, by independence of $\epsilon_i$ and $a(x,X_i)$, we have
\begin{equation*}
\begin{split}
    \mathbb{E}[\sum_{i=1}^n\epsilon_ia(x,X_i)]=\sum_{i=1}^n \mathbb{E}[\epsilon_i]\mathbb{E}[a(x,X_i)]=0
\end{split}
\end{equation*} 
Hence the variables $\sum_{i=1}^n[f(X_i)-b_n(f,x)]a(x,X_i)$ and $\sum_{i=1}^n\epsilon_ia(x,X_i)$ as sums of i.i.d variables are good candidates for concentration inequalities. We recall \textit{Bernstein's concentration inequality} for bounded distributions (\cite{vershynin} Theorem 2.8.4, page 39)).
Note that for $A,B>0$ the function
\begin{equation}
\label{convenience_note}
    s\to\exp{(-\frac{A}{s+B})}
\end{equation}
is increasing on $[0,\infty)$, so we will use the following version of Bernstein's inequality
\begin{theorem}{\textbf{(Bernstein's inequality for bounded distributions})}
\label{BERNSTEIN}
Suppose that $Z_1,Z_2,...Z_n$ are independent, centered and such that $|Z_i|\leq K$. Then for every $t\geq 0$ and every $s\geq \sum_{i=1}^n \mathbb{E}[Z_i^2]$ we have
\begin{equation*}
\mathbb{P}(|\sum_{i=1}^n Z_i|\geq t)\leq 2\exp(-\frac{t^2/2}{s+Kt/3})    
\end{equation*}
\end{theorem}
The first observation is that the \textit{empirical} degree $\sum_{i=1}^n a(x,X_i)$ can be replaced by its \textit{polpulation} version, i.e. the local degree \eqref{local_degree} $d_n(x)$, provided that
$d_n(x)$ is not too small. This is formally stated as the following lemma.
\begin{lemma}
\label{bernstein_corollary}
\begin{equation*}
    \mathbb{P}(|\sum_{i=1}^na(x,X_i)-d_n(x)|\geq \frac{d_n(x)}{2})\leq 2\exp({-\frac{3d_n(x)}{14}}) 
\end{equation*}
\end{lemma}
\begin{proof}
We apply Theorem \ref{BERNSTEIN} with the variables $Z_i=a(x,X_i)-c_n(x)$. Clearly,
\begin{equation*}
-1\leq -c_n(x)\leq Z_i\leq 1-c_n(x)\leq 1    
\end{equation*}
For all $i=1,2,...,n$ we have
\begin{equation*}
    \mathbb{E}[Z_i^2]=\mathbb{E}[(a(x,X_i)-c_n(x))^2]=c_n(x)(1-c_n(x))\leq c_n(x)
\end{equation*}
Hence
\begin{equation*}
    \sum_{i=1}^n\mathbb{E}[Z_i^2]=\sum_{i=1}^n \mathbb{E}[(a(x,X_i)-c_n(x))^2]\leq nc_n(x)=d_n(x)
\end{equation*}
Setting $t=\frac{d_n(x)}{2}$ and $s=d_n(x)$ we get
\begin{equation*}
\begin{split}
    \mathbb{P}(|\sum_{i=1}^n a(x,X_i)-d_n(x)|\geq \frac{d_n(x)}{2})&\leq 2\exp(-\frac{d^2_n(x)/4}{d_n(x)+d_n(x)/6})\\
    &=2\exp(-3d_n(x)/14)
\end{split}
\end{equation*}
\end{proof}
Lemma \ref{bernstein_corollary} states that the concentration is \textit{exponential} in the local degree $d_n(x)$. 
The rest of this section aims to prove that with a bounded noise assumption (so that \textit{Bernstein's} inequality applies), 
$\hat{f}_{GNW}(x)$ deviates from $b_n(f,x)$ at a rate exponentialy decreasing in $d_n(x)$. Towards the end of the section we discuss
rates obtainable under other noise assumptions. The exponential rate in $d_n(x)$ in the absence of noise is proven in the following lemma.
\begin{lemma}
\label{basic_lemma_1}
Suppose that $f$ is bounded, measurable function with  $||f||_{\infty}\leq B$. Then 
\begin{equation*}
\mathbb{P}(|\frac{\sum_{i=1}^nf(X_i)a(x,X_i)}{d_n(x)}-\frac{\sum_{i=1}^na(x,X_i)}{d_n(x)}b_n(f,x)|\geq \delta)\leq 2\exp(-\frac{2\delta^2d_n(x)}{4B^2+B\delta/3})
\end{equation*}
\end{lemma}
\begin{proof}
We set $Z_i=[f(X_i)-b_n(f,x)]a(x,X_i)$, $i=1,2,3,...,n$. Then $Z_i$ form a sequence of i.i.d., centered variables with $|Z_i|\leq 2B$.
The sum of their variances satisfies 
\begin{equation*}
\begin{split}
    \sum_{i=1}^n \mathbb{E}[Z_i^2]=\sum_{i=1}^n\mathbb{E}[(f(X_i)-b_n(f,x))^2a(x,X_i)]&=n\mathbb{E}[[f(X_1)-b_n(f,x)]^2a(x,X_i)]\\
    &\leq 4nB^2c_n(x)=4B^2d_n(x)
\end{split}
\end{equation*}
By Theorem \ref{BERNSTEIN} $s=4B^2d_n(x)$ we get that for every $t\geq 0$
\begin{equation*}
    \mathbb{P}(|\sum_{i=1}^n[f(X_i)-b_n(f,x)]a(x,X_i)|\geq t)\leq 2\exp(\frac{-t^2/2}{4B^2d_n(x)+Bt/3})
\end{equation*}
Substituting $t=\delta d_n(x)$ gives the desired result.
\end{proof}
The following result bounds the probability of noise deviating at rate exponentially decreasing in $d_n(x)$ and is yet another similar application of Bernstein's inequality.
For that reason, we ommit the proof.
\begin{lemma}
\label{basic_lemma_2}
Suppose that the noise variables are bounded and centered, i.e. $|\epsilon_i|\leq \sigma$. Then
\begin{equation*}
    \mathbb{P}(|\frac{\sum_{i=1}^n \epsilon_{i}a(x,X_i)}{d_n(x)}|\geq \delta)\leq \exp(-3\delta^2d_n(x)/(2\sigma+6\delta\sigma^2))
\end{equation*}
\end{lemma}
Combining the basic lemmas from this section, we get the following result.
\begin{theorem}
\label{concentration_thm}
Suppose that $f\colon\mathbb{R}^d\to\mathbb{R}$ is bounded with $||f||_{\infty}\leq B$ and the noise variables satisfy $|\epsilon_i|\leq\sigma$.
Then 
\begin{equation*}
  \mathbb{P}(|\hat{f}_{GNW}(x)-b_n(f,x)|\geq \delta)\leq 6\exp(-C(\delta,B,\sigma)d_n(x))
\end{equation*}
where $C(\delta,B,\sigma)=\min\{3/14,3\delta^2/(32\sigma+96\sigma^2),6\delta^2/(192B^2+\delta B))$
\end{theorem}
\begin{proof}
Suppose that
\begin{equation}
\label{deg_cond}
    \sum_{i=1}^n a(x,X_i)\geq d_n(x)/2
\end{equation}
Note that in particular this implies $\sum_{i=1}^n a(x,X_i)>0$.
Then
\begin{equation}
\begin{split}
    |\hat{f}_{GNW}(x)-b_n(f,x)|&\leq \frac{2|\sum_{i=1}^nf(X_i)a(x,X_i)-b_n(f,x)\sum_{i=1}^n a(x,X_i)|}{d_n(x)}\\
    &+\frac{2|\sum_{i=1}^n \epsilon_ia(x,X_i)|}{d_n(x)}
\end{split}
\end{equation}
If in addition
\begin{equation}
\label{weird_cond}
|\frac{\sum_{i=1}^nf(X_i)a(x,X_i)}{d_n(x)}-\frac{\sum_{i=1}^na(x.X_i)}{d_n(x)}b_n(f,x)|<\delta/4
\end{equation}
and
\begin{equation}
\label{noise_cond}
|\sum_{i=1}^n \epsilon_ia(x,X_i)|<\delta/4
\end{equation}
Then we get 
\begin{equation*}
|\hat{f}_{GNW}(x)-b_n(f,x)|<\delta
\end{equation*}
Hence if $|\hat{f}_{GNW}(x)-b_n(f,x)|\geq \delta$ then at least one of the inequalities (\ref{deg_cond}), (\ref{weird_cond}) or (\ref{noise_cond}) must be violated. We conclude by using Lemmas \ref{bernstein_corollary}, \ref{basic_lemma_1} and \ref{basic_lemma_2} together with a union bound 
\begin{equation*}
\begin{split}
    \mathbb{P}(|\hat{f}_{GNW}(x)-b_n(f,x)|\geq\delta)&\leq \mathbb{P}(|\sum_{i=1}^n a(x,X_i)-d_n(x)|\geq d_n(x)/2)\\
    &+\mathbb{P}(|\frac{\sum_{i=1}^nf(X_i)a(x,X_i)}{d_n(x)}-\frac{b_n(f,x)\sum_{i=1}^na(x,X_i)}{d_n(x)}|\geq \delta/4)\\
    &+\mathbb{P}(|\frac{\sum_{i=1}^n \epsilon_i a(x,X_i)}{d_n(x)}|\geq \delta/4)\\
    &\leq 6\exp(-C(\delta,B,\sigma)d_n(x))
\end{split}
\end{equation*}
where $C(\delta,B,\sigma)=\min\{3/14,3\delta^2/(32\sigma+96\sigma^2),6\delta^2/(192B^2+\delta B))$
\end{proof}
\begin{figure}
    \centering
    \includegraphics[width=0.7\textwidth]{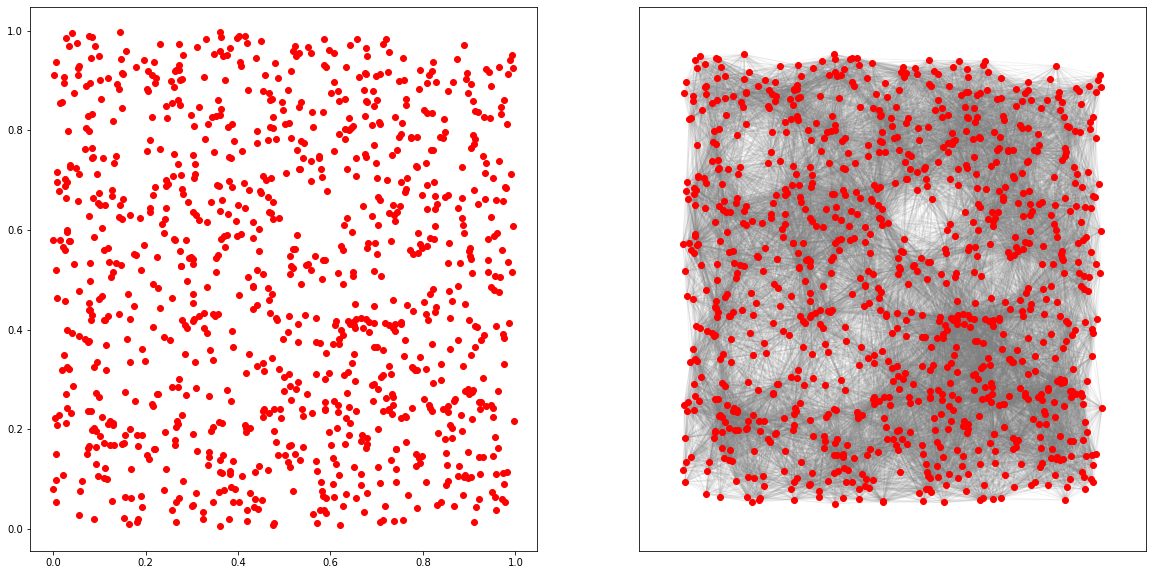}
    \caption{Random Geometric Graph on $n=1000$ points, with average degree $\sqrt{n}$}
    \label{fig:dense_rgg}
\end{figure}
We conclude this section with several remarks.
\begin{remark}
Theorem \ref{concentration_thm} states that if $d_n(x)\to\infty$ as $n\to\infty$ then 
$$\hat{f}_{GNW}(x)-b_n(f,x)\to 0$$ in probability. In particular there is no requirement on the rate of growth of $d_n(x)$ to ensure convergence in probability. On the other hand, if for all $n\in\mathbb{N}$, $d_n(x)\geq r\log(n)$ for some $r>1$, then the left hand side of the inequality in Theorem \ref{concentration_thm} is summable, so the Borel Cantelli lemma implies that $\hat{f}_{GNW}(x)-b_n(f,x)\to 0$ almost surely.
\end{remark}
\begin{remark}
\label{remark_gaus}
The fact that the noise variables $\epsilon_1,...,\epsilon_n$ are bounded was crucial in obtaining such 
a strong bound as in Theorem \ref{concentration_thm}. 
If we assume that the noise is Gaussian  we 
\textbf{can prove} a significantly worse bound on the probability appearing in \ref{concentration_thm}, namely one of order $O(\exp(-c(\delta,B,\sigma)d_n^2(x)/n))$. This limitation is due to the classical concentration inequalities developed for gaussian variables. 
It states that with gaussian noise, convergence in probability in ensured only in the case when we have $d_n(x)/\sqrt{n}\to\infty$, i.e. where the expected local degree at $x$ grows faster than $\sqrt{n}$ (see figure \ref{fig:dense_rgg}).
We do not know if this is just a technical limitation or if this is the true threshold for (sub)gaussian noise.
\end{remark}

\begin{remark}
\label{remark_cheb}
If we assume only finite second moments of the noise, i.e. $\mathbb{E}[\epsilon^2]=\sigma^2<\infty$, then an application of Chebyshev's
inequality, one can show that 

\begin{equation*}
\mathbb{P}(|\hat{f}_{GNW}(x)-b_n(f,x)|\geq \delta)\leq \frac{C(B,\sigma^2)}{d_n(x)}    
\end{equation*}
This result motivates us to expect a bound on the variance proxy \eqref{variance_term} $v_n(x)$ 
of order $1/d_n(x)$ as well. We prove that this is indeed the case in the following section.
\end{remark}

\section{Sharp Variance Bounds}
\label{sharp_variance_bounds}
In Theorem \ref{concentration_thm} and Remarks \ref{remark_gaus} and \ref{remark_cheb} we showed that if $d_n(x)$ is large, $\hat{f}_{GNW}(x)$  will concentrate towards $b_n(f,x)$
depending on the type of noise assumption. In this section we adopt the weakest assumption on the noise, that of finite variance i.e. $\mathbb{E}[\epsilon_1^2]=\sigma^2<\infty$.
The goal of this section is to prove a bound of the following form
\begin{equation*}
\label{variance_thm_informal}
    v_n(x)=\mathbb{E}[\hat{f}_{GNW}(x)-b_n(f,x)]^2\leq\frac{C(B,\sigma^2)}{d_n(x)}
\end{equation*}
where $C(B,\sigma^2)$ depends on the boundedness constant $B$ of $f$ and the variance of the noise $\sigma^2$. 
This is achieved in Theorem \ref{variance_thm}.
\subsection{The decoupling trick}
\label{variance_bounds}
In order to tackle the variance term \eqref{variance_term}, we use a  
\textbf{decoupling trick} that brings independence into the weights of $\hat{f}_{GNW}$ which otherwise are \textit{ratios} of dependent variables.  
For $I\subseteq [n]$, let
\begin{equation*}
\label{eqn_R_empty}
R_I(x)= \begin{cases}
    \frac{1}{|I|+\sum_{j\notin I}a(x,X_j)}, \hspace{3pt} I\neq\emptyset \\
        \frac{1}{\sum_{i=1}^n a(x,X_i)}, \hspace{3pt} I=\emptyset \hspace{3pt} \text{and} \hspace{3pt} \sum_{i=1}^n a(x,X_i)>0\\
        0, \hspace{3pt} \text{otherwise}\\
\end{cases}
\end{equation*}
For convenience of notation we write

\begin{equation*}
    R_i(x)=R_{\{i\}}(x)
\end{equation*}
and 
\begin{equation}
\label{shorthand}
    Z=\mathbb{I}(\sum_{i=1}^n a(x,X_i)>0)
\end{equation}
Taking into account the fact that $a(x,X_i)$ is a Bernoulli variable, i.e. it takes values $0$ and $1$, it follows that for all $i=1,2,...,n$ 
\begin{equation}
\label{eqn_R_single}
R_{\emptyset}(x)a(x,X_i)=R_{i}(x)a(x,X_i)
\end{equation}
Indeed, if $a(x,X_i)=0$ then both sides of Equation (\ref{eqn_R_single}) are $0$. Otherwise $a(x,X_i)=1$ and both sides in Equation (\ref{eqn_R_single}) equal $R_{i}(x)$. Moreover, $R_i(x)$ is independent from $a(x,X_i)$.
More generally we have the following observation.
\begin{lemma}{\textbf{(Decoupling trick)}}
    \label{DECOUPLING}
    For all pairs of \textbf{disjoint} subsets $I,J\subseteq$[n] we have
    \begin{equation*}
    R_J(x)\prod_{i\in I}a(x,X_i)=R_{I\cup J}(x)\prod_{i\in I}a(x,X_i)
    \end{equation*}
    and $R_{I\cup J}(x)$ is independent from $\{a(x,X_i)|i\in I\}$.
\end{lemma}
\begin{proof}
If $\prod_{i\in I}a(x,X_i)=0$ then there is nothing to prove. On the other hand, if $\prod_{i\in I}a(x,X_i)\neq 0$, then by the fact that $a(x,X_i)$ are Bernoulli variables we get $a(x,X_i)=1$ for all $i\in I$.
Hence as $I\subseteq[n]-J$, we have
\begin{equation*}
\label{decoupling_trick}
    R_{J}(x)=\frac{1}{|J|+\sum_{i\notin J}a(x,X_i)}=\frac{1}{|I|+|J|+\sum_{i\notin I\cup J}a(x,X_i)}=R_{I\cup J}(x)
\end{equation*}
The second part of the lemma follows from modelling assumptions.
\end{proof}
Even though elementary, this observation introduces independence into the weights of GNW, which are \textit{ratios} 
of dependent random variables. It will be used to prove a sharp bound on the variance proxy \eqref{variance_term} $v_n(x)$ as well as for 
computing the expectation $\mathbb{E}[\hat{f}_{GNW}(x)]$. The following calculations are preparation for Theorem \ref{trick_lemma_pt2}. In particular, the following two lemmas 
show how to \textit{decouple} the variables $Z$ and $\hat{f}_{GNW}(x)$. 
\begin{lemma} We have
    \label{eqn_for_Ri}
        \begin{equation*}
            \sum_{i=1}^n a(x,X_i)R_i(x)=Z
        \end{equation*}
\end{lemma}
\begin{proof}
    Summing over $i=1,2,3,...,n$ in Equation  (\ref{eqn_R_single}) gives
    
    \begin{equation*}
    \begin{split}
    \sum_{i=1}^n a(x,X_i)R_i(x)=\sum_{i=1}^n a(x,X_i)R_{\emptyset}(x)=\mathbb{I}(\sum_{i=1}^n a(x,X_i)>0)=Z
    \end{split}
    \end{equation*}
\end{proof}
Note that $\hat{f}_{GNW}(x)=\sum_{i=1}^n Y_ia(x,X_i)R_{\emptyset}(x)$. Using Equation (\ref{eqn_R_single}), we get the following 
result.
\begin{lemma} We have
\label{eqn_for_gnw_Ri}
    \begin{equation*}
        \hat{f}_{GNW}(x)=\sum_{i=1}^n Y_ia(x,X_i)R_i(x)
    \end{equation*}
\end{lemma}
Note that by Definition \eqref{gnw_def} and Equation \eqref{shorthand}

\begin{equation}
\label{gnw_no_edges}
 (1-Z)\hat{f}_{GNW}(x)=0
\end{equation}
or equivalently
\begin{equation}
\label{gnw_edges}
   Z\hat{f}_{GNW}(x)=\hat{f}_{GNW}(x) 
\end{equation}
Keeping in mind that $Z$ is Bernoulli random variable, we have $Z^2=Z$ and $(1-Z)^2=1-Z$, so using Equations (\ref{gnw_no_edges}, \ref{gnw_edges}), we get 
\begin{equation}
\label{variance_decomp}
\begin{split}
    v_n(x)&=\mathbb{E}[(\hat{f}_{GNW}(x)-b_n(f,x))^2Z]+\mathbb{E}[(\hat{f}_{GNW}(x)-b_n(f,x))^2(1-Z)]\\
    &=\mathbb{E}[(\hat{f}_{GNW}(x)Z-b_n(f,x)Z)^2]+\mathbb{E}[(\hat{f}_{GNW}(x)(1-Z)-b_n(f,x)(1-Z))^2]\\
    &=\mathbb{E}[\hat{f}_{GNW}(x)-b_n(f,x)Z]^2+b_n^2(f,x)\mathbb{E}[1-Z]\\
    &=\mathbb{E}[\hat{f}_{GNW}(x)-b_n(f,x)Z]^2+b_n^2(f,x)(1-c_n(x))^n
\end{split}
\end{equation}
From Equation \eqref{variance_decomp} it follows that we only need to focus on control of 
\begin{equation*}
    \mathbb{E}[(\hat{f}_{GNW}(x)-b_n(f,x)Z)^2]
\end{equation*}
With this goal in mind, using Lemmas \ref{eqn_for_gnw_Ri} and \ref{eqn_for_Ri} we have
\begin{equation}
\label{decomp}
\begin{split}
    \hat{f}_{GNW}(x)-b_n(f,x)Z&=
    \sum_{i=1}^nY_ia(x,X_i)R_i(x)-b_n(f,x)\sum_{i=1}^na(x,X_i)R_i(x)\\
    &=\sum_{i=1}^n(Y_i-b_n(f,x))a(x,X_i)R_i(x)
\end{split}
\end{equation}
We will show that the summands in the right hand side of Equation \eqref{decomp} are uncorrelated and consequently we will obtain tractable expression for $\mathbb{E}[(\hat{f}_{GNW}(x)-b_n(f,x)Z)^2]$. We first state a preliminary lemma which follows easily from Lemma \ref{DECOUPLING}.
\begin{lemma}
\label{one_time_lemma}
Suppose that $g\colon\mathbb{R}^{d+1}\to\mathbb{R}$ is a  measurable function such that \\ $g(X_1,\epsilon_1)\in\mathbb{L}^2$. For $1\leq i\leq n$ set $F_i=g(X_i,\epsilon_i)$. Then for all pairs of distinct indices $(i,j)$, $1\leq i,j \leq n$ we have
\begin{equation*}
    \mathbb{E}[F_iF_ja(x,X_i)a(x,X_j)R_i(x)R_j(x)]=\mathbb{E}[F_ia(x,X_i)]\mathbb{E}[F_ja(x,X_j)]\mathbb{E}[R_{\{i,j\}}(x)^2]
\end{equation*}
\end{lemma}
\begin{proof}
Using the decoupling trick \ref{DECOUPLING} we have
\begin{equation*}
    F_iF_ja(x,X_i)a(x,X_j)R_i(x)R_j(x)=F_iF_ja(x,X_i)a(x,X_j)R_{\{i,j\}}(x)^2
\end{equation*}
and moreover $R_{\{i,j\}}(x)$ is independent from $(X_i,\epsilon_i,a(x,X_i))$ and $(X_j,\epsilon_j,a(x,X_j))$. Next, $(X_i,\epsilon_i,a(x,X_i))$ and $(X_j,\epsilon_j,a(x,X_j))$ are also independent by modeling assumption. As independence implies uncorrelatedness, the conclusion follows.
\end{proof}
\begin{lemma}
\label{trick_lemma_pt_1}
We have
\begin{equation*}
    \mathbb{E}[(\hat{f}_{GNW}(x)-b_n(f,x)Z)^2]=
    n[\mathbb{E}[(f(X_1)-b_n(f,x))^2a(x,X_i)]+\sigma^2d_n(x)\mathbb{E}[R_1^2(x)]
\end{equation*}
\end{lemma}
\begin{proof}
Set $g(X_i,\epsilon_i)=Y_i-b_n(f,x)$.
Using Equation (\ref{decomp}), we have

\begin{equation}
\label{tricky_eqn}
\begin{split}
    \mathbb{E}[(\hat{f}_{GNW}(x)-b_n(f,x)Z)^2]&=\mathbb{E}[(\sum_{i=1}^n(Y_i-b_n(f,x))a(x,X_i)R_i(x))^2]\\
    &=\sum_{i=1}^n\mathbb{E}[(g(X_i,\epsilon_i)a(x,X_i)R_i(x))]^2\\
    &+\sum_{i\neq j}\mathbb{E}[g(X_i,\epsilon_i)g(X_j,\epsilon_j)a(x,X_i)a(x,X_j)R_i(x)R_j(x)]
\end{split}
\end{equation}
For $i\neq j$, applying Lemma \ref{one_time_lemma} with  $g\colon\mathbb{R}^{d+1}\to\mathbb{R}$ given by 
$g(\cdot,*)=(f(\cdot )+*)-b_n(f,x)$ along with the fact that $\mathbb{E}[g(X_i,\epsilon_i)a(x,X_i)]=0$ gives

\begin{equation}
\label{dr_trick}
    \mathbb{E}[(Y_i-b_n(f,x))(Y_j-b_n(f,x))a(x,X_i)a(x,X_j)R_i(x)R_j(x)]=0
\end{equation}
Furthermore,
\begin{equation*}
\label{sr_trick}
\begin{split}
    \sum_{i=1}^n \mathbb{E}[(g(X_i,\epsilon_i)a(x,X_i)R_i(x))^2]&=\sum_{i=1}^n \mathbb{E}[(Y_i-b_n(f,x))^2a(x,X_i)]E(R_i^2(x))\\
    &=n\mathbb{E}[(Y_1-b_n(f,x))^2a(x,X_i)]\mathbb{E}[R_1^2(x)]\\
    &=n(\mathbb{E}[(f(X_1)-b_n(f,x))^2a(x,X_i)]+\sigma^2c_n(x))\mathbb{E}[R_1^2(x)]
\end{split}
\end{equation*}
\end{proof}

\subsection{Upper bounds}
After the long preparation, we finally prove a 
bound on $\mathbb{E}[(\hat{f}_{GNW}(x)-b_n(f,x)Z)^2]$.
The following lemma is crucial towards an upper bound in the variance proxy \eqref{variance_term}. 
We recall that $Z=\mathbb{I}(\sum_{i=1}^na(x,X_i)>0)$.
\begin{theorem} 
\label{trick_lemma_pt2}
For $n\geq 3$ and $f$ a bounded measurable function with $||f||_{\infty}\leq B$, we have
\begin{equation*}
    \mathbb{E}[(\hat{f}_{GNW}(x)-b_n(f,x)Z)^2]\leq \frac{260B^2+65\sigma^2}{d_n(x)} 
\end{equation*}
\end{theorem}

\begin{proof}
Recalling Lemma \ref{trick_lemma_pt_1} and using the fact that $||f||_{\infty}\leq B$, we have 

\begin{equation}
\label{ubv_1}
     n[\mathbb{E}[(f(X_1)-b_n(f,x))^2a(x,X_i)]+\sigma^2c_n(x)]\mathbb{E}[R_1^2(x)]
     \leq (4B^2+\sigma^2)nc_n(x)\mathbb{E}[R^2_1(x)]
\end{equation}
Hence it suffices to control $\mathbb{E}[R_1^2(x)]$. We do this by splitting the expectation on the event that we observe at least $\frac{1}{2}(n-1)c_n(x)$ edges from $a(x,X_i)$, $i=2,...,n$ and on it's complement. On the event that we observe at least $\frac{1}{2}(n-1)c_n(x)$ edges $R_1$ will be bounded from above by a quantity of order $\frac{C}{d_n(x)}$, for an explicit constant $C>0$. The main observation is that $R_1\leq 1$ and observing too few edges is an event with small probability. The rest of the proof deals with technical calculations. Let 
\begin{equation}
    A(x)=\{\sum_{i=2}^na(x,X_i)\geq \frac{1}{2}(n-1)c_n(x)\}
\end{equation}
For $n\geq 2$ we have

\begin{equation}
\label{meat_good_part}
\mathbb{E}[R_1^2(x)\mathbb{I}(A(x))]\leq \frac{1}{(1+\frac{1}{2}(n-1)c_n(x))^2}\mathbb{P}(A(x))\leq \frac{16}{n^2c_n^2(x)}
\end{equation}
We apply Bernstein's inequality for bounded distributions \ref{BERNSTEIN} with
$$Z_i=a(x,X_i)-c_n(x)$$
$i=2,3,...n$ as the bounded, centered and independent variables.
\begin{equation*}
    \sum_{i=2}^n\mathbb{E}[Z_i^2]=(n-1)c_n(x)(1-c_n(x))\leq (n-1)c_n(x)
\end{equation*}
Setting $s=(n-1)c_n(x)$ we get
\begin{equation}
\label{bernstein_result}
\mathbb{P}(|\sum_{i=2}^{n}a(x,X_i)-(n-1)c_n(x)|\geq t)\leq 2\exp{(-\frac{t^2/2}{(n-1)c_n(x)+t/3})}
\end{equation}
Setting $t=\frac{1}{2}(n-1)c_n(x)$ in Equation (\ref{bernstein_result}) together with the observation that $A^c(x)$ implies 
\begin{equation*}
    |\sum_{i=2}^{n}(a(x,X_i)-c_n(x))|\geq \frac{1}{2}(n-1)c_n(x)
\end{equation*}
we get
\begin{equation}
\label{meat_bad_part_anticip}
\begin{split}
    \mathbb{P}(A^c(x))&\leq \mathbb{P}(|\sum_{i=2}^{n}[a(x,X_i)-c_n(x)]|\geq \frac{1}{2}(n-1)c_n(x))\\
    &\leq \exp(-\frac{3(n-1)c_n(x)}{14})\\
    &\leq \exp(-\frac{nc_n(x)}{7})
\end{split}
\end{equation}
Using the fact that $R_1\leq 1$ along with Equation (\ref{meat_bad_part_anticip}) we get
\begin{equation}
\label{meat_bad_part}
    \mathbb{E}[R^2_1(x)\mathbb{I}(A^c(x))]\leq \mathbb{P}(A^c(x))\leq \exp(-\frac{nc_n(x)}{7})
\end{equation}
Combining Equation (\ref{meat_good_part}) and Equation (\ref{meat_bad_part}) gives

\begin{equation*}
\mathbb{E}[R^2_1(x)]\leq \frac{16}{n^2c_n^2(x)}+\exp(-\frac{nc_n(x)}{7})
\end{equation*}
Finally,
\begin{equation}
    \label{meat}
    d_n(x)\mathbb{E}[R^2_1(x)]\leq \frac{16}{d_n(x)}+d_n(x)\exp(-\frac{d_n(x)}{7})
\end{equation}
The conclusion follows by combining Equations \eqref{ubv_1} and  \eqref{meat}, together with the basic inequality which states that for all $x\geq 0$, $x^2e^{-x}\leq 1$.
\end{proof}

\begin{theorem}{(\textbf{Sharp Variance Bound})}
\label{variance_thm}
Suppose that $f$ is a bounded measurable function with $||f||_{\infty}\leq B$ and $\mathbb{E}[\epsilon_1^2]=\sigma^2$. Then
\begin{equation*}
    v_n(x)=\mathbb{E}[(\hat{f}_{GNW}(x)-b_n(f,x))^2]\leq \frac{261B^2+65\sigma^2}{d_n(x)}
\end{equation*}
\end{theorem}

\begin{proof}

Using Equation (\ref{variance_decomp}),
Lemma \ref{trick_lemma_pt2} and using the basic inequality $1-t\leq \exp{(-t)}$ valid for all $t\geq 0$, we get

\begin{equation*}
\begin{split}
    \mathbb{E}[(\hat{f}_{GNW}(x)-b_n(f,x))^2]
    &=\mathbb{E}[(\hat{f}_{GNW}(x)-b_n(f,x)Z)^2]+b_n^2(f,x)\mathbb{P}(\sum_{i=1}^n a(x,X_i)=0)\\
    &\leq (\frac{260B^2+65\sigma^2}{d_n(x)})+B^2(1-c_n(x))^n\\
    &\leq \frac{260B^2+65\sigma^2}{d_n(x)}+B^2\exp(-d_n(x))
\end{split}
\end{equation*}
Note that this result is slightly stronger than the statement of the theorem. For simplicity, we bound the second term by the dominating term $\frac{1}{d_n(x)}$.
We conclude by using the basic inequality:
for all $t\geq 0$, $te^{-t}\leq 1$.
\end{proof}
\begin{figure}[h!]
    \centering
    \includegraphics[width=0.7\textwidth]{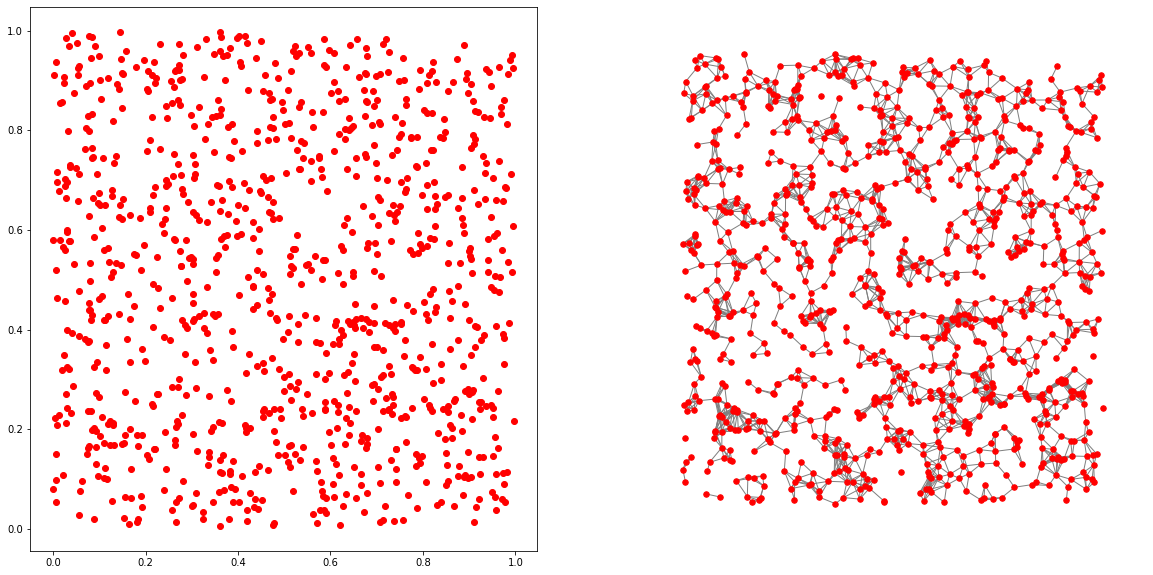}
    \caption{A Random Geometric Graph with $n=1000$ points and average degree $d_n(x)=\log(\log(n))$. Theorems \ref{concentration_thm} and   \ref{variance_thm} imply that $\hat{f}_{GNW}(x)$ concentrates towards $b_n(f,x)$ even in this asymptotic regime of growth}
    \label{fig:sparse_rgg}
\end{figure}

\subsection{Lower bounds}
Here we compliment upper bound given in Theorem \ref{variance_thm} by an (almost) matching lower bound of order $1/d_n(x)$, valid in the presence of noise.
\begin{lemma} 
\label{variance_lower_bound} Suppose that $f$ is a bounded measurable function with $||f||_{\infty}\leq B$ and $\mathbb{E}[\epsilon_1^2]=\sigma^2<\infty$. Then  
\begin{equation*}
    \mathbb{E}[(\hat{f}_{GNW}(x)-b_n(f,x))^2]\geq \frac{\sigma^2(1-e^{-d_n(x)})^2}{d_n(x)}
\end{equation*}
\end{lemma}

\begin{proof}
By Equation (\ref{variance_decomp}), Lemma
\ref{expectation_comp}, Lemma  
\ref{trick_lemma_pt_1} and the basic inequality  $1-t\leq e^{-t}$ valid for all $t\geq 0$, we have

\begin{equation}
\begin{split}
\mathbb{E}[(\hat{f}_{GNW}(x)-b_n(f,x))^2]&\geq \mathbb{E}[(\hat{f}_{GNW}(x)-b_n(f,x)Z)^2]\\
&=n[\mathbb{E}[(f(X_1)-b_n(f,x))^2a(x,X_i)]+\sigma^2c_n(x)]\mathbb{E}[R_1^2(x)]\\
&\geq \sigma^2nc_n(x)\mathbb{E}[R_1^2(x)]\\
&\geq \frac{\sigma^2(1-(1-c_n(x))^n)^2}{nc_n(x)}\\
&\geq \frac{\sigma^2(1-e^{-nc_n(x)})^2}{nc_n(x)}
\end{split}
\end{equation}
\end{proof}

\section{Expectation of GNW}
\label{expectation_computation}

Being a quotient of two random variables, the exact value of $\mathbb{E}[\hat{f}_{GNW}(x)]$ may seem difficult to compute. This is done via the decoupling trick \ref{DECOUPLING}.
\begin{lemma}
\label{lemma_exp}
For all $i-1,2,...,n$ we have
\begin{equation*}
    \mathbb{E}[R_i(x)]=\frac{1-(1-\frac{d_n(x)}{n})^n}{d_n(x)}
\end{equation*}
\end{lemma}
\begin{proof}
Note that $R_i(x)$, $i=1,2,...,n$ are identically distributed, hence  $\mathbb{E}[R_i(x)]=\mathbb{E}[R_1(x)]$ for $i=2,...,n$.
By Lemma \ref{eqn_for_Ri} we have
\begin{equation}
    \sum_{i=1}^n a(x,X_i)R_i(x)=Z
\end{equation}
Taking expectation and using the fact that $R_i(x)$ and $a(x,X_i)$ are independent, we get
\begin{equation}
\label{eqn_for_ERi}
\begin{split}
    \mathbb{E}[\sum_{i=1}^n a(x,X_i)R_i(x)]&=\sum_{i=1}^n \mathbb{E}[a(x,X_i)R_i(x)]\\
    &=\sum_{i=1}^n \mathbb{E}[a(x,X_i)]\mathbb{E}[R_i(x)]\\
    &=nc_n(x)\mathbb{E}[R_1(x)]
\end{split}
\end{equation}
On the other hand,
\begin{equation}
\label{eqn_for_Rempty}
    \mathbb{E}[Z]=\mathbb{P}(\sum_{i=1}^n a(x,X_i)>0)=1-\mathbb{P}(\sum_{i=1}^n a(x,X_i)=0)=1-(1-c_n(x))^n
\end{equation}
The result follows by combining Equations (\ref{eqn_for_Ri}), (\ref{eqn_for_ERi}) and (\ref{eqn_for_Rempty}).
\end{proof}

\begin{theorem}{(\textbf{Computation of $\mathbb{E}[\hat{f}_{GNW}(x)]$})}
\label{expectation_comp}
\begin{equation*}
    \mathbb{E}[\hat{f}_{GNW}(x)]=b_n(f,x)(1-(1-c_n(x))^n)
\end{equation*}
\end{theorem}
\begin{proof}

By Lemma (\ref{eqn_for_gnw_Ri}) we have

\begin{equation*}
    \hat{f}_{GNW}(x)=\sum_{i=1}^n Y_ia(x,X_i)R_i(x)
\end{equation*}

Hence, taking expectation and using Lemma \ref{lemma_exp}, we get 

\begin{equation*}
\begin{split}
    \mathbb{E}[\hat{f}_{GNW}(x)]&=\sum_{i=1}^n\mathbb{E}[Y_ia(x,X_i)R_i(x)]\\
    &=\sum_{i=1}^n \mathbb{E}[Y_ia(x,X_i)]\mathbb{E}[R_i(x)]\\
    &=n\mathbb{E}[Y_1a(x,X_1)]\mathbb{E}[R_1(x)]\\
    &=\frac{T_{k_n}(f)(x)(1-(1-c_n(x))^n)}{c_n(x)}\\
    &=b_n(f,x)(1-(1-c_n(x))^n)
\end{split}
\end{equation*}

\end{proof}

\begin{corollary}
\label{proxies}   
    Let $\Bias_n[\hat{f}_{GNW}(x)]$ and $\Var_n[\hat{f}_{GNW}(x)]$ denote the \textbf{standard} bias and variance of $\hat{f}_{GNW}(x)$, i.e.
\begin{equation*}
\begin{split}
    &\Bias_n[\hat{f}_{GNW}(x)]=\mathbb{E}[\hat{f}_{GNW}(x)]-f(x) \hspace{3pt} \text{and}\\
    &\Var_n[\hat{f}_{GNW}(x)]=\mathbb{E}[(\hat{f}_{GNW}(x)-\mathbb{E}[\hat{f}_{GNW}(x)])^2]
\end{split}
\end{equation*} 
If $||f||_{\infty}\leq B$, then  
\begin{equation}
\begin{split}
    v_n(x)-\Var_n(\hat{f}_{GNW}(x))&=(b_n(x)-\Bias_n(\hat{f}_{GNW}(x)))^2\\
    &=b^2_n(f,x)(1-\frac{d_n(x)}{n})^{2n}\\
    &\leq B^2\exp(-2d_n(x))
\end{split}
\end{equation}
\end{corollary}
\begin{proof}
In view of Proposition \ref{expectation_comp}, we have
\begin{equation}
\label{temp_lab_1}
    b_n(x)-\Bias_n(\hat{f}_{GNW}(x))=b_n(f,x)-\mathbb{E}[\hat{f}_{GNW}(x)]=b_n(f,x)(1-\frac{d_n(x)}{n})^n
\end{equation}
Next,
\begin{equation}
\label{temp_lab_2}
\begin{split}
    v_n(x)-\Var_n(\hat{f}_{GNW}(x))&=\mathbb{E}[(\hat{f}_{GNW}(x)-b_n(f,x))^2-(\hat{f}_{GNW}(x)-\mathbb{E}[\hat{f}_{GNW}(x)])^2]\\
    &=[\mathbb{E}[\hat{f}_{GNW}(x)]-b_n(f,x)]\mathbb{E}[(2\hat{f}_{GNW}(x)-b_n(f,x)-\mathbb{E}[\hat{f}_{GNW}(x)])]\\
    &=(b_n(f,x)-\mathbb{E}[\hat{f}_{GNW}(x)])^2\\
    &=b^2_n(f,x)(1-\frac{d_n(x)}{n})^{2n}
\end{split}
\end{equation}
The claim follows from Equations \eqref{temp_lab_1} and\eqref{temp_lab_2} and the basic inequality $1-t\leq \exp(-t)$.
\end{proof}
We conclude this section with several remarks.
\begin{remark} There is an interesting property of the consistency of $\hat{f}_{GNW}(x)$. Suppose that $d_n(x)\to\infty$ as $n\to\infty$. Then Theorem (\ref{variance_thm}) implies
\begin{equation*}
    \mathbb{E}[(\hat{f}_{GNW}(x)-b_n(f,x))^2]\to 0
\end{equation*}
On the other hand if $d_n(x)\leq D$ for all $n\in\mathbb{N}$ and $\sigma^2>0$ then Lemma \ref{variance_lower_bound} gives 
\begin{equation*}
    \mathbb{E}[(\hat{f}_{GNW}(x)-b_n(f,x))^2]\geq \frac{\sigma^2(1-e^{-D})}{D}>0
\end{equation*}
\end{remark}
\begin{remark} In the bounded degree regime, $\hat{f}_{GNW}(x)$ is not asymptotically unbiased even for the simplest functions.
Consider $f$ to be the constant function $1$. Then $b_n(f,x)=1$. If $d_n(x)\leq d_0$ then 
$$(1-\frac{d_n(x)}{n})^n\geq (1-\frac{d_0}{n})^n$$
and consequently
\begin{equation*}
    f(x)-\mathbb{E}[\hat{f}_{GNW}(x)]=(1-\frac{d_n(x)}{n})^n\geq (1-\frac{d_0}{n})^n
\end{equation*}
in particular 
\begin{equation*}
    \liminf_{n\to\infty}({f(x)-\mathbb{E}[\hat{f}_{GNW}(x)]})\geq 1-e^{-d_0}
\end{equation*}
\end{remark}
\chapter{Bias and Risk of GNW}
\label{bias_risk_control}

In Chapter \ref{chap2}, we considered a LPM graph where we established the fact that under bounded noise assumption the Graphical Nadaraya Watson estimator $\hat{f}_{GNW}(x)$ 
concentrates towards $b_n(f,x)$ at a rate exponentially decreasing in the local degree \eqref{local_degree} $d_n(x)$ (Theorem \ref{concentration_thm}). We also bounded the 
variance proxy \eqref{variance_term} by a quantity of order $1/d_n(x)$ (Theorem \ref{variance_thm}). 
We recall that
\begin{equation*}
    k_n(x,z)=\alpha_nK(\frac{x-z}{h_n})
\end{equation*}
with $0<\alpha_n\leq 1$ and $h_n>0$.
In this section we address the following questions:
\begin{enumerate}
    \item  Under which conditions is $b_n(f,x)$ a good approximation of $f(x)$?
    \item  How does $d_n(x)$ depend on the parameters $\alpha_n$ and $h_n$?
\end{enumerate}
Indeed, so far nothing has been said about the 
bias proxy \eqref{bias_term}. Question 1 is treated in Section \ref{bias_section}. Question 2 is 
treated in Section \ref{degree_and_bw}. Once these questions are addresed, we will be able to comment on the pointwise \eqref{formula} 
and integrated \eqref{random_point_risk} risks in terms of the parameters $\alpha_n$ and $h_n$. This is done in 
Sections \ref{pwrisk} and \ref{irisk}, respectively.

\section{Uniform bound on the Bias}
\label{bias_section}
In order to control the bias proxy \eqref{bias_term} we will need
to assume regularity conditions on the regression function $f$, the kernel function $K$ and on the density $p$. 
\begin{assumption}
\label{K_1}
There exists $M_1>0$ for all $z\in\mathbb{R}^d$
\begin{equation*}
    \frac{1}{2}\mathbb{I}(||z||\leq M_1)\leq K(z)
\end{equation*}
\end{assumption}
\begin{assumption}
\label{K_2}
There exists $M_2>0$ such that for all $z\in\mathbb{R}^d$
\begin{equation*}
    K(z)\leq \mathbb{I}(||z||\leq M_2)
\end{equation*}
\end{assumption}
These assumptions are a generalization of the Random Geometric Graph. It can be shown that under Assumptions \ref{K_1} and \ref{K_2}, if $h_n\to 0$ and $nh_n^d\to\infty$, the integrated 
risk \eqref{integrated_risk_nw} of the NW estimator \eqref{NW} satisfies $\mathcal{R}(\hat{f}_{NW},f)\to 0$ \cite{Gyofri}. 
Assumption \ref{K_1} is relatively weak, as it can be replaced by continuity at of $K$ at $z=0$ and $K(0)=1$. It will be important in understanding how the expected degree $d_n(x)$ relates to $h_n$. 
Assumption \ref{K_2} says that $K$ has compact support. 
It will be important for controlling the bias proxy \eqref{bias_term}.
As our ultimate goal is a bound on the \textit{integrated} risk \eqref{formula}, we will focus on points in the \textit{support} $p$, $Q=\supp{(p)}$. 
We assume that $f$ is $a$-Hölder continuous on $Q$. 
\begin{assumption}
\label{F_1}
There exist $0<a\leq 1$ and $L>0$ such that for all $x,z\in Q$
\begin{equation*}
    |f(x)-f(z)|\leq L||x-z||^a
\end{equation*}
Equivalently, $f\in\Sigma(a,L)$ on $Q$.
\end{assumption}
The following lemma shows that under Assumption \ref{K_1} the variance term (\ref{variance_term}) is nontrivial at points in $Q$.
\begin{lemma}
\label{well_cond_lemma}
Suppose that Assumption \ref{K_1} holds. Then
\begin{equation*}
    Q\subseteq \{x\in\mathbb{R}^d:c_n(x)>0\}
\end{equation*}
\end{lemma}
\begin{proof}
Suppose that $c_n(x)=0$.
Using Assumption \ref{K_1} we get 
\begin{equation*}
\begin{split}
    \alpha_n\int \mathbb{I}(||x-z||\leq M_1h_n) p(z)dz &\leq 2\alpha_n\int \mathbb{I}(||x-z||\leq M_1h_n)K(\frac{x-z}{h_n})p(z)dz\\
    &\leq 2\alpha_n\int K(\frac{x-z}{h_n})p(z)dz\\
    &=2c_n(x)\\
    &=0
\end{split}
\end{equation*}
As $\alpha_n>0$, $x\notin\supp{p}$, proving the claim by contraposition.
\end{proof}
The following lemma states that if we fix $Q$ then Assumptions \ref{K_2} and \ref{F_1} are strong enough to guarantee uniform bound over all 
distributions $p$ with $\supp{p}=Q$ \textbf{and} over $x\in Q$.
\begin{lemma}
\label{bias_control_lemma}
Suppose that Assumptions \ref{K_2} and \ref{F_1} hold. Let $p$ be a density function with $\supp{(p)}=Q$. Then
    \begin{equation*}
    \sup_{x\in Q}|b_n(f,x)-f(x)|\leq 2LM_2^{\alpha}h_n^{\alpha}
\end{equation*}
\end{lemma}
\begin{proof}
For $x\in Q$ by Lemma \ref{well_cond_lemma}, $c_n(x)=\int K(\frac{x-z}{h_n})p(z)dz>0$.
We have 
\begin{equation*}
\begin{split}
|b_n(f,x)-f(x)|&=|\frac{\int f(z)K(\frac{x-z}{h_n})p(z)dz}{\int K(\frac{x-z}{h_n})p(z)dz}-f(x)|\\
&=|\frac{\int f(z)K(\frac{x-z}{h_n})p(z)dz}{\int K(\frac{x-z}{h_n})p(z)dz}-\frac{\int f(x)K(\frac{x-z}{h_n})p(z)dz}{\int K(\frac{x-z}{h_n})p(z)dz}|\\
&=|\frac{\int_{Q}[f(z)-f(x)]K(\frac{x-z}{h_n})p(z)dz
}{\int_{Q} K(\frac{x-z}{h_n})p(z)dz}|\\
&\leq L\frac{ \int_{Q} ||z-x||^{\alpha}K(\frac{x-z}{h_n})p(z)dz}{
\int_{Q} K(\frac{x-z}{h_n})p(z)dz}\\
&\leq 2LM_2^{\alpha}h_n^{\alpha}
\end{split}
\end{equation*}
where we used Assumption \ref{F_1} in the first and  Assumption \ref{K_2} in the last inequality.
\end{proof}
We remark that the sparsity factor $\alpha_n$ does not play a role in the bound on the bias proxy \eqref{bias_term}. However, it will play a significant role 
for the variance proxy \eqref{variance_term}.
\section{Local degree in terms of bandwitdh and sparsity parameters}
\label{degree_and_bw}
Under Assumptions \ref{K_2} and \ref{F_1}, the bias proxy \eqref{bias_term} is already uniformly bounded over $Q$ by Lemma (\ref{bias_control_lemma}). 
By Theorem (\ref{variance_thm}) it suffices to bound $\frac{1}{d_n(x)}$. Observe that under Assumption \ref{K_1} we have that for all $x\in Q$ 
\begin{equation*}
    c_n(x)=\alpha_n\int K(\frac{x-z}{h_n})p(z)dz\geq \alpha_n\int \mathbb{I}(|x-z|\leq M_1h_n)p(z)dz
\end{equation*}
Before we make stronger assumptions on the density $p$, we show that kernel Assumption \ref{K_1} guarantees that as soon as $n\alpha_nh_n^d\to\infty$ as $n\to\infty$, the variance proxy \ref{variance_term} converges to zero for almost every $x\in\mathbb{R}^d$.
\begin{lemma}
\label{Lebesgue_lemma}
Suppose that Assumptions \ref{K_1} and \ref{K_2} hold. Let $v_d$ be the Lebesgue measure of the unit ball in $\mathbb{R}^d$. Then almost everywhere with respect to  Lebesgue measure,
\begin{equation*}
 \frac{1}{2}v_dM_1^dp(x)\leq \liminf_{n\to\infty} \frac{d_n(x)}{n\alpha_nh_n^d}\leq \limsup_{n\to\infty} \frac{d_n(x)}{n\alpha_nh_n^d}\leq v_dM_2^dp(x)   
\end{equation*}
\end{lemma}
\begin{proof}
Using Assumptions \ref{K_1} and \ref{K_2} we have
\begin{equation}
\label{lebesgue_density_pts}
    \frac{1}{2}\int\mathbb{I}(|x-z|\leq M_1h_n)p(z)dz\leq \frac{c_n(x)}{\alpha_n}\leq \int\mathbb{I}(|x-z|\leq M_1h_n)p(z)dz
\end{equation}
By Lebesgue's differentiation theorem (\cite{Stein:1385521} Theorem 1.4, page 106) we have 

\begin{equation}
\frac{\int\mathbb{I}(|x-z|\leq M_1h_n)p(z)dz}{h_n^d}\to\ v_dM_1^dp(x)
\end{equation}
and
\begin{equation}
\frac{\int\mathbb{I}(|x-z|\leq M_2h_n)p(z)dz}{h_n^d}\to v_dM_2^dp(x)
\end{equation}
Dividing the inequality (\ref{lebesgue_density_pts}) by $h_n^d$ and letting $n\to\infty$ we get the desired result.
\end{proof}
Lemma \ref{lebesgue_density_pts} is not useful for our task of bounding the pointwise risk \eqref{pointwise_risk}, 
but it does give good heuristic for the relationship between $d_n(x)$, $\alpha_n$ and $h_n$.
In order to turn this heuristic into a nonasymptotic bound on $1/d_n(x)$, we introduce the following definition.
We say that $G\subseteq\mathbb{R}^d$ has $(r_0,c_0)$-\textbf{measure-retaining property} if
for all $x\in G$ and all $r\leq r_0$, 
\begin{equation*}
    m(G\cap B_r(x))\geq c_0m(B_r(x))
\end{equation*}
Here $m$ is the Lebesgue measure on $\mathbb{R}^d$.
\begin{assumption}
\label{mrl} There exist $r_0,c_0>0$ such that $Q=\supp{(p)}$ has $(r_0,c_0)-$measure-retaining property.
\end{assumption}

Clearly, $\mathbb{R}^d$ has $(\infty,1)$- measure-retaining property. It is not difficult to show that 
the Cube $Q_d=[-1,1]^d$ has the $(1,\frac{1}{2^d})-$measure-retaining property, and so does every closed 
and convex subset of $\mathbb{R}^d$ (for some $r_0,c_0>0$). Another broad class of sets which satisfy this property and are used in the regression context in $\mathbb{R}^d$ 
are those that satisfy \textit{interior cone condition}. A set $Q$ satisfies an interior cone condition with cone $C$ if for all points $x\in Q$, one can rotate and translate $C$
to a cone $C_x$ with a vertex in $x$ such that $C_x\subseteq Q$. The following result shows that under measure-retention assumption, the only problematic points for the local degree $d_n(x)$ and consequently the variance proxy \eqref{variance_thm} are those in whose 
neighbourhood, the density $p$ is low.
\begin{lemma}
\label{local_risk_lemma} Suppose that Assumption \ref{K_1} and \ref{mrl} hold. 
If $M_1h_n<r_0$ and $x\in Q$ is such that 
\begin{equation}
\label{lbd}
\inf\limits_{\substack{z\in Q\\|x-z|\leq M_1h_n}} p(z)\geq p_0(x)>0
\end{equation}
Then 
\begin{equation*}
    \frac{1}{d_n(x)}\leq \frac{2}{c_0v_dM_1^dn\alpha_nh_n^dp_0(x)}
\end{equation*}
\end{lemma}
\begin{proof}
By Assumption \ref{K_1} and the assumption that $Q$ has the $(r_0,c_0)-$measure retaining property we have 
\begin{equation*}
\begin{split}
\frac{c_n(x)}{\alpha_n}&=\int K(\frac{x-z}{h_n})p(z)dz\\
&\geq \frac{1}{2}\int \mathbb{I}(|x-z|\leq M_1h_n)p(z)dz\\
&\geq \frac{p_0(x)}{2}\int \mathbb{I}(|x-z|\leq M_1h_n)\mathbb{I}(z\in Q)dz\\
&=\frac{p_0(x)}{2}m(Q\cap B_{M_1h_n}(x))\\
&\geq \frac{p_0(x)c_0}{2}m(B_{M_1h_n}(x))\\
&=c_0v_dM_1^dh_n^dp_0(x)/2
\end{split}
\end{equation*}
Hence
\begin{equation*}
    \frac{1}{d_n(x)}=\frac{1}{nc_n(x)}\leq \frac{2}{c_0v_dM_1^dn\alpha_nh_n^dp_0(x)}
\end{equation*}
\end{proof}
\begin{corollary}
\label{local_variance_control}
Suppose that Assumption \ref{K_1}, \ref{mrl} and Equation \eqref{lbd} hold. Then
\begin{equation*}
    v_n(x)\leq \frac{522B^2+130\sigma^2}{c_0v_dM_1^dn\alpha_nh_n^dp_0(x)}
\end{equation*}
\end{corollary}
\begin{proof}
The statement follows immediately from Theorem \ref{variance_thm} and Lemma \ref{local_risk_lemma}
\end{proof}
\section{Pointwise risk}
\label{pwrisk}
Having established bounds on the bias \eqref{bias_term} and variance \eqref{variance_term} proxies, we are ready to 
provide a bound on the pointwise risk \eqref{formula}. 

\begin{theorem}{(\textbf{Pointwise risk bound})}
\label{pwriskthm}    
Suppose that Assumptions \ref{K_1}, \ref{K_2}, \ref{F_1}, \ref{mrl} and Equation \eqref{lbd} hold. 
If $M_1h_n\leq r_0$ then
\begin{equation*}
        \mathcal{R}(\hat{f}_{GNW}(x),f(x))\leq 4L^2M_2^{2a}h_n^{2a}+\frac{1044B^2+260\sigma^2}{c_0v_dM_1^dn\alpha_nh_n^dp_0(x)}
\end{equation*}
\end{theorem}

\begin{proof}
We use the bias and variance proxies to bound the risk via the following inequality   
    \begin{equation}
    \label{quasi_bias_variance}
    \mathcal{R}(\hat{f}_{GNW}(x),f(x))\leq 2(v_n(x)+b_n^2(x))
    \end{equation}
On one hand, from Lemma \ref{bias_control_lemma} we see that under Assumptions \ref{K_2} and \ref{F_1}, we have 

\begin{equation*}
|b_n(x)|\leq 2LM_2^ah_n^a    
\end{equation*}

On the other hand, from Lemma \ref{local_variance_control} we see that under Assumptions \ref{K_1}, \ref{mrl} and Equation \eqref{lbd} we have 
\begin{equation*}
    v_n(x)\leq \frac{522B^2+130\sigma^2}{c_0v_dM_1^dn\alpha_nh_n^dp_0(x)}
\end{equation*}
The conclusion follows form Equation \eqref{quasi_bias_variance}
\end{proof}
We remark that rates of this form are well known for the NW \eqref{NW} estimator \cite{Tsybakov,Gyofri}.
\section{Integrated risk}
\label{irisk}
Finally to bound the integrated risk \eqref{random_point_risk} of GNW, we would like to integrate the inequality given in 
Theorem \ref{pwriskthm}. Unfortunately the right hand side of this inequality depends on $p_0(x)$, a quantity that depends nontrivially 
on the behavior of $p$ around the point $x\in Q$, so a direct integration is still not an option. An easy way to fix this problem is to make the following assumption.
\begin{assumption}
\label{density_assumption}
There exists $p_0>0$ such that for all $x\in Q$,
\begin{equation*}
    p(x)\geq p_0
\end{equation*}
\end{assumption}
\begin{theorem}
\label{final_result}
Suppose that Assumptions \ref{K_1}, \ref{K_2}, \ref{F_1}, \ref{mrl} and \ref{density_assumption} hold.
If $M_1h_n<r_0$, we have
\begin{equation*}
    \mathcal{R}(\hat{f}_{GNW},f)\leq 4L^2M_2^{2a}h_n^{2a}+\frac{1044B^2+260\sigma^2}{p_0c_0v_dM_1^dn\alpha_nh_n^d}
\end{equation*}
\end{theorem}
\begin{proof}
Under Assumption \ref{density_assumption}, Equation \eqref{lbd} holds with $p_0(x)\equiv p_0$. 
The conclusion follows immediately from Equation \eqref{random_point_risk} and Theorem \ref{pwriskthm}    
\end{proof}
Let us comment on the assumptions we have made so far. The kernel assumptions (Assumption \ref{K_1} and \ref{K_2}) say that edges can 
occur only between nodes whose latent positions have distance from one another which is less than a 
certain threshold. The regularity Assumption \ref{F_1} come as natural limitations from the simplicity of the estimator. 
The Assumption \ref{mrl} deals with the boundary issues. Finally, the density assumption \ref{density_assumption} says that there are no low density regions in $Q$. In particular,
it implies that the $Q$ has finite Lebesgue measure. We will now replace the restrictive assumption \ref{density_assumption} on the density $p$ by Hölder continuity. This allows us to treat densities that are supported on the entirety of 
$\mathbb{R}^d$, such as the gaussian density. 
\begin{assumption} There exist $0<b\leq 1$ and $L>0$ such that $p\in\Sigma(b,L)$ and
\label{hcd_condition}
    \begin{equation*}
        \int p^{1/2}(x)dx<\infty
    \end{equation*}
\end{assumption}
The cost of this assumption is the considerably slower rate in terms of $h_n$ given in the next theorem.
\begin{theorem}
\label{final_result_holder}
    Suppose that Assumptions \ref{K_1}, \ref{K_2}, \ref{F_1}, \ref{mrl} and \ref{hcd_condition} hold. If $h_n<\min{(r_0/M_1,1)}$ then 
    \begin{equation*}
        \mathcal{R}(\hat{f}_{GNW},f)\leq C_1h_n^{\min{(2a,\beta/2)}}+\frac{C_2}{n\alpha_nh_n^{d+\beta}}
    \end{equation*}
where $C_1=\max{(4L^2M_2^{2a},4B^2L^{1/2}M_1^{\beta/2}\int p^{1/2}(x)dx)}$, $C_2=\frac{1044B^2+260\sigma^2}{c_0v_dLM_1^{d+\beta}}$.
\end{theorem}
\begin{proof}
    Under these assumptions we have
    \begin{equation*}
        \mathcal{R}(\hat{f}_{GNW}(x),f(x))\leq 2\min\{b^2_n(x)+v_n(x),2B^2\}
    \end{equation*}
From Assumption \ref{hcd_condition} we have 
\begin{equation*}
    \inf\limits_{\substack{z\in Q\\|x-z|\leq M_1h_n}} p(z)\geq p(x)-LM_1^{\beta}h_n^{\beta}
\end{equation*}
The idea now is to split the integral in the integrated risk \eqref{random_point_risk} in two parts, the first where 
the density is sufficiently high ($\geq 2LM_1^{\beta}h_n^{\beta}$), where we use the bounds from Theorem \ref{pwriskthm}
and the second, where the density is low and on which we use the bound $\mathcal{R}(\hat{f}_{GNW}(x),f(x))\leq 4B^2$.
We have

\begin{equation*}
\begin{split}
   \mathcal{R}(\hat{f}_{GNW},f)&=\int\mathcal{R}(\hat{f}_{GNW}(x),f(x))p(x)dx\\
    &\leq\int\limits_{\{p(x)\geq 2LM_1^{\beta}h_n^{\beta}\}}\mathcal{R}(\hat{f}_{GNW}(x),f(x)) p(x)dx+4B^2\int\limits_{\{p(x)\leq 2M_1^{\beta}h_n^{\beta}\}}p(x)dx\\
    &\leq 4L^2M_2^{2a}h_n^{2a}+\frac{1044B^2+260\sigma^2}{c_0v_dLM_1^{d+\beta}n\alpha_nh_n^{d+\beta}}+4B^2L^{1/2}M_1^{\beta/2}h_n^{\beta/2}\int p^{1/2}(x)dx 
\end{split}    
\end{equation*}
\end{proof}

\section{Discussion}
We recall again that for GNW $\alpha_n$ and $h_n$ are not tunable parameters. 
This discussion aims to describe a range of values for $\alpha_n$ and $h_n$ on which the integrated 
risk of GNW achieves error of order $1/n^{r}$. The following discussion will be able to cover 
Theorem \ref{final_result} and Theorem  \ref{final_result_holder} at the same time. We assume that $n\alpha_n\to\infty$. For $C_1,C_2,\gamma,\Delta>0$ consider the expression 

\begin{equation}
\label{rate_analisys}
    F(h_n)=C_1h_n^{\gamma}+\frac{C_2}{n\alpha_nh_n^{\Delta}}
\end{equation}
For sutable choices of $C_1,C_2,\gamma$ and $\Delta$, one can replicate the rates obtained in Theorems \ref{final_result} and \ref{final_result_holder}.
Setting each summand in Equation \eqref{rate_analisys} to be less than $\epsilon/2$, we get that the interval in \eqref{range_of_vals} is non-degenerate, i.e.
if
\begin{equation}
\label{range_of_vals}
    (\frac{2C_2}{n\alpha_n\epsilon})^{\frac{1}{\Delta}}\leq h_n \leq (\frac{\epsilon}{2C_1})^{\frac{1}{\gamma}}
\end{equation}
then $F(h_n)\leq \epsilon$. We note that as $\epsilon$ decreases, the interval \eqref{range_of_vals} shrinks. In particular, for $0<r<\frac{\gamma}{\Delta+\gamma}$ 
and $\epsilon=2C_1^{\frac{\Delta}{\Delta+\gamma}}C_2^{\frac{\gamma}{\Delta+\gamma}}(n\alpha_n)^{-r}$, we have
that if 

\begin{equation*}
    (\frac{C_2}{C_1})^{\frac{1}{\Delta+\gamma}}\frac{1}{(n\alpha_n)^{\frac{1-r}{\Delta}}}\leq h_n\leq (\frac{C_2}{C_1})^{\frac{1}{\Delta+\gamma}} \frac{1}{(n\alpha_n)^{\frac{r}{\gamma}}}
\end{equation*}
then 
\begin{equation}
\label{rate_bound}
    F(h_n)\leq \frac{2C_1^{\frac{\Delta}{\Delta+\gamma}}C_2^{\frac{\gamma}{\Delta+\gamma}}}{(n\alpha_n)^{r}}\leq \frac{2(\frac{\Delta}{\Delta+\gamma}C_1+\frac{\gamma}{\Delta+\gamma}C_2)}{(n\alpha_n)^{r}}
\end{equation}

In particular, when $r=\frac{\gamma}{\Delta+\gamma}$, the interval in \eqref{range_of_vals} shrinks to a point, and the error rate 
$F(h_n)$ is optimized. Replacing suitable parameters for $C_1,C_2,\gamma,\Delta$, we get the following two results.

\begin{theorem}
Suppose that Assumptions \ref{K_1}, \ref{K_2}, \ref{F_1}, \ref{mrl} and \ref{density_assumption} hold.
Set $C_1=4L^2M_2^{2a}$, $C_2=\frac{1044B^2+260\sigma^2}{p_0c_0v_dM_1^d}$ and $0<r\leq \frac{2a+d}{d}$. If $M_1h_n<r_0$ and 

\begin{equation*}
    \frac{1}{(n\alpha_n)^\frac{1-r}{d}}\leq (\frac{C_1}{C_2})^{d+2a} h_n\leq \frac{1}{(n\alpha_n)^{\frac{r}{2a}}}
\end{equation*}
then 
\begin{equation*}
    (n\alpha_n)^r\mathcal{R}(\hat{f}_{GNW},f)\leq 8(\frac{dL^2M_2^{2a}}{d+2a}+\frac{a(522B^2+130\sigma^2)}{p_0c_0(d+2a)v_dM_1^d})  
\end{equation*}
\end{theorem}

\begin{theorem}
Suppose that Assumptions \ref{K_1}, \ref{K_2}, \ref{F_1}, \ref{mrl} and \ref{hcd_condition} hold. 
\\
Set $0<r\leq \frac{\min{(2a,\beta/2)}}{d+\beta+\min{(2a,\beta/2)}}$,
\begin{equation*}
\begin{split}
    C_1&=\max{(4L^2M_2^{2a},4B^2L^{1/2}M_1^{\beta/2}\int p^{1/2}(x)dx)} \hspace{3pt} \text{and}\\
    C_2&=\frac{1044B^2+260\sigma^2}{c_0v_dLM_1^{d+\beta}}
\end{split}    
\end{equation*}
If $h_n<\min{(r_0/M_1,1)}$ and
    \begin{equation*}
        (n\alpha_n)^{-\frac{1-r}{d+\beta}}\leq(\frac{C_1}{C_2})^{d+\beta+\min(2a,\beta/2)}h_n\leq (n\alpha_n)^{-\frac{r}{\min{(2a,\beta/2)}}}
    \end{equation*}
    then 
    \begin{equation*}
    \begin{split}
        (n\alpha_n)^{r}\mathcal{R}(\hat{f}_{GNW},f)&\leq \frac{8(d+\beta)\max{(L^2M_2^{2a},B^2L^{1/2}M_1^{\beta/2}\int p^{1/2}(x)dx)}}{d+\beta+\min{(2a,\beta/2)}}\\
        &+\frac{8\min{(2a,\beta/2)}(261B^2+65\sigma^2)}{c_0(d+\beta+\min{(2a,\beta/2)})v_dLM^{d+\beta}}
    \end{split}    
\end{equation*}
\end{theorem}

\begin{remark} \textbf{Bias-variance tradeoff}
Note that as $h_n$ decreases, the bias proxy \ref{bias_term} decreases, 
but the bounds on the variance proxy \ref{variance_term} grows. Conversely, 
as $h_n$ increases, the variance proxy \ref{variance_term} decreases but the 
bound \ref{bias_control_lemma} on the bias proxy \ref{bias_term} grows. 
This is also the case with the classical Nadaraya Watson estimator and is a 
general phenomenon in statistics known as the \textbf{bias-variance trade off}.
\end{remark}

\begin{remark} \textbf{The Curse of Dimensionality}
According to Stirling's approximation, the volume $v_d$ of the unit ball in $\mathbb{R}^d$ scales like $v_d\sim \frac{1}{\sqrt{d\pi}}(\frac{2\pi e}{d})^{d/2}$. As a consequence of this, it follows that $n\alpha_n$ should grow  
exponentially with $d$, i.e. to ensure integrated risk less than $\epsilon$, $n\alpha_n$ will grow exponentially in $d$. This is the \textbf{Curse of Dimensionality}, another well known phenomenon in statistics \cite{Gyofri}. 
\end{remark}

\begin{remark} \textbf{Functions and densities of higher regularity}
    In \cite{Tsybakov} it is shown that, for \textbf{univariate} regression functions of higher regularity (achieved by demanding Hölder continuity of the derivatives of $f$), one can achieve faster rates for the NW estimator \eqref{NW}. In particular,
    one can generalize the Hölder classes to $\Sigma(a,L)$ with $a>1$, and one can show a result of the form 
    \begin{equation*}
        \mathcal{R}(\hat{f}_{NW}(x),f(x))\leq C_1h_n^{2a}+\frac{C_2}{nh_n}
    \end{equation*}
which can be optimized in $h_n$ to get a \textbf{minimax} results for the pointwise risk
 of the form
 \begin{equation*}
    \inf_{h_n>0}\sup_{f\in\Sigma(\alpha,L)}\mathcal{R}(\hat{f}_{NW}(x),f(x))\leq Cn^{-\frac{2\alpha}{2\alpha+1}}
 \end{equation*}
The optimal bandwith is of the  form $cn^{\frac{-\alpha}{2\alpha+1}}$. When the denisty $p$ satisfies Assumption \ref{density_assumption}, then one has the same rate for the integrated risk as well,
\begin{equation*}
    \inf_{h_n>0}\sup_{f\in\Sigma(\alpha,L)}\mathcal{R}(\hat{f}_{NW},f)\leq Cn^{-\frac{2\alpha}{2\alpha+1}}
\end{equation*}
However, as sharper rates on the risk require  
symmetry conditions on the kernel which would be restrictive in our setup where the kernel is \textbf{not known}, we decide to not pursue results of this kind.
\end{remark}
\chapter*{Conclusion} 
We showed that both the \textit{pointwise} and \textit{integrated} risk bounds of the risk of $\hat{f}_{GNW}$ are similar to ones of the classical NW  
estimator. If the graph comes from a LPM with \textit{latent} bandwith $h_n$, then the performance 
of GNW \eqref{gnw_def} is comparable to the corresponding NW \eqref{NW} estimator with the \textit{fixed} bandwith $h_n$. If $h_n$ falls into the suitable range 
of values (i.e. $h_n\to 0$ and $n\alpha_nh_n^d\to\infty$) then GNW will perform well. As GNW uses only one-hop neighbourhood information, it does not take advantage of the 
global graph structure, it would be interesting to compare it with \textit{graph} spectral based regression estimators (such as graphical Kernel Ridge Regression). It  would also be interesting to understand if estimating the 
latent positions could be statistically beneficial for estimation.

\small\printbibliography
\end{document}